\RequirePackage{amsmath}
\documentclass[12pt]{llncs}
\usepackage{llncsdoc}
\usepackage[margin=1in]{geometry}
\usepackage{beton}
\usepackage[T1]{fontenc}
\usepackage{cite}
\usepackage{wrapfig}
\usepackage[shortlabels]{enumitem}
\usepackage{amssymb}
\usepackage{algorithmicx}
\usepackage{algorithm}
\usepackage{mathtools}
\usepackage{hyperref}
\hypersetup{
    colorlinks=true, 
    linktoc=all,     
    linkcolor=blue,  
}

\spnewtheorem{myclaim}{Claim}{\bfseries}{\itshape}
\newtheorem{cor}[theorem]{Corollary}
\newtheorem{defn}[theorem]{Definition}
\newtheorem{fact}[theorem]{Fact}

\newcommand{\argmax}{\operatornamewithlimits{argmax}}
\newcommand{\argmin}{\operatornamewithlimits{argmin}}
\newcommand\norm[1]{\left\lVert#1\right\rVert}
\newcommand\wideRightarrow{\mathrel{\scalebox{1.5}[1]{$\Rightarrow$}}}
\newcommand{\dit}{\delta_i^t}
\newcommand{\dito}{\delta_i^{t-1}}

\newcommand{\dott}[2]{\left\langle#1 \,\middle\vert\,  #2 \right\rangle}
\newcommand{\myprop}{\ensuremath{\texttt{RVU}}~}
\newcommand{\vt}[3]{\vec{#1}_{#2}^{#3}}
\renewcommand{\vec}[1]{\ensuremath{{\bf #1}}}
\renewcommand{\mod}[1]{\left| #1 \right|}

\begin{document}
\title{Dynamic Pricing in Competitive Markets}
\author{Paresh Nakhe}
\institute{Institute of Computer Science, Goethe University,\\ Frankfurt (Main), Germany.\\
\email{Nakhe@em.uni-frankfurt.de}
}
\maketitle

\begin{abstract}
Dynamic pricing of goods in a competitive environment to maximize revenue is a natural objective and has been a subject of research over the years. In this paper, we focus on a class of markets exhibiting the substitutes property with sellers having divisible and replenishable goods. Depending on the prices chosen, each seller observes a certain demand which is satisfied subject to the supply constraint. The goal of the seller is to price her good dynamically so as to maximize her revenue. For the static market case, when the consumer utility satisfies the Constant Elasticity of Substitution (CES) property, we give a $O(\sqrt{T})$ regret bound on the maximum loss in revenue of a seller using a modified version of the celebrated Online Gradient Descent Algorithm by Zinkevich~\cite{Zinkevich}.  For a more specialized set of consumer utilities satisfying the iso-elasticity condition, we show that when each seller uses a regret-minimizing algorithm satisfying a certain technical property, the regret with respect to $(1-\alpha)$ times optimal revenue is bounded as $O(T^{1/4} / \sqrt{\alpha})$. We extend this result to markets with dynamic supplies and prove a corresponding dynamic regret bound, whose guarantee deteriorates smoothly with the inherent instability of the market. As a side-result, we also extend the previously known convergence results of these algorithms in a general game to the dynamic setting.
\end{abstract}

\section{Introduction}

Internet has revolutionized the way goods are bought and sold and has in the process created a range of new possibilities to price the goods strategically and dynamically. This is especially true for online retail and apparel stores where the cost and effort to update prices has become negligible. This flexibility in pricing has propelled the research in \textit{dynamic pricing} in the last decade or so, informally defined as the study of determining optimal selling prices in an unknown environment to optimize an objective, usually revenue. Coupled with the presence of digitally available and frequently updated sales data one may also view this as an (online) learning problem.

The inherent hurdles in dynamic pricing arise on account of \textit{lack of information}. In the context of a single good case, this could be the underlying demand function that maps a given price to the observed demand. Indeed, this problem has been studied in several models in literature and strong results are now known for it. However, the problem becomes all the more challenging in a realistic setting where multiple sellers independently choose prices for their goods and the demand observed by any single seller is a function of all the prices. For example, some fixed seller might observe completely different demands for the same price she uses for her items depending on the prices chosen by other sellers. Such a seller might falsely conclude of being in a dynamic environment even when the underlying demand function is static.

Several existing approaches for dynamic pricing assume a parametric form for the underlying demand function and choose a sequence of prices to learn the individual parameters by statistical estimation. This approach is commonly referred to as ``learn-and-earn" in literature~\cite{harrison2012bayesian, KZ2}. It would, however, be unrealistic in the presence of multiple sellers since that would imply learning highly nonlinear and possibly unstructured functions in high dimensions. Instead, we view the market as a set of strategic agents (the sellers) choosing successive actions (prices) in order to maximize their utility (revenue) and focus on using the existing rich tool-kit of  \textit{agnostic learning} in game theoretic models to prove fast convergence to optimal prices.

The advantages of an agnostic learning approach are multifold: Firstly, it does not rely on the precise parametric form of the underlying demand function and secondly can be easily extended to the case when the market parameters may change across rounds. The downside, however, being that in the best case of static markets with clean parametric representation, the algorithms might converge to optimal prices only asymptotically~\cite{Keskin2014, Mertikopoulos2016}. Consequently, to measure the performance of the actions (prices) chosen by such a learning algorithm we typically compare it to a certain benchmark sequence of actions and the \textit{regret bound} represents the loss incurred by the algorithm for not having chosen the benchmark sequence instead. In most such algorithms, this benchmark sequence is usually a single action that gives the maximum cumulative utility over all rounds.

We base our dynamic pricing approach on the work by Syrgkanis et al \cite{Syrgkanis} where the authors prove that in a game with multiple agents, if each agent uses a regret-minimizing algorithm with a suitable step-size parameter and satisfying a certain technical property, then the individual regret of each agent is bounded by $O(T^{1/4})$ where $T$ is the total number of rounds. Although the main result is proved in the discrete action setting, the authors show that the same technique can be extended to agents with continuous action sets as well. In a nutshell, these algorithms \textit{anticipate} the utility vector for the forthcoming round and choose a price such that the cumulative utility over all previous rounds and the forthcoming one is maximized. The regret bound thus obtained holds with respect to the single best price in hindsight and is one of the benchmarks we use to measure the performance of our approach.

\noindent
\textbf{Contribution:} Our contributions in this paper can be broadly divided into 3 parts: 
\begin{enumerate}[1)]
\item In the first part, for the class of markets with gross substitutes CES utility functions, we show that a simple modification to the Online Gradient Descent (OGD) algorithm by Zinkevich \cite{Zinkevich} can be used to obtain a regret bound on the loss in revenue with respect to the single best price in hindsight of order $O(\sqrt{T})$. We note that CES utilities represent a very general market model used in both economics and CS literature.
\item Following the analysis in \cite{Syrgkanis} and using a \textit{smoothed} revenue objective we obtain a stronger regret bound of order $O\left(  \frac{T^{1/4}}{\sqrt{ \alpha}} \right)$ against a $(1 - \alpha)$ multiplicative approximation of the best price in hindsight. This analysis, however, holds only for the more restricted class of iso-elastic markets.
\item In the last part, we extend the technical property introduced in \cite{Syrgkanis} to the case of dynamic regret i.e. when the performance of the algorithm is compared to a certain benchmark sequence. For the class of iso-elastic markets we show the existence of learning algorithms satisfying this property and use them to prove a regret bound of order $ O\left( (1 + W_T)\left(  \frac{T^{1/4}}{\sqrt{  \alpha} } \right) \right)$ against a $(1 - \alpha)$ multiplicative approximation of the benchmark prices. Here $W_T$ is a measure of the inherent instability of the market. As a side result, we can use this new property in ~\cite{Syrgkanis} to prove bounds for dynamic regret.
\end{enumerate}

A key observation in this work is that if the sellers in a market are ready to let go of a small fraction of their revenue, then they can converge to their (approximately) optimal prices (in static market setting) much faster ($T^{-3/4}$ instead of $T^{-1/2}$). This faster convergence property is all the more desirable when the markets \textit{drift} and convergence to optimal strategy in a small number of rounds is not possible. One would then like to achieve good performance with respect to a dynamic benchmark.

\subsection*{Related Work}

The problem of learning an optimal pricing policy for various demand models and inventory constraints has been researched extensively in the last decade. However, many consider the problem of a single good with no \textit{competition effects}. \cite{Keskin2014, carvalho2005learning, den2013simultaneously, Broder2012, besbes2009dynamic} study a parametric family of demand functions and design an optimal pricing policy by estimating the unknown parameters by standard techniques such as linear regression or maximum likelihood estimation. \cite{Kleinberg, besbes2011minimax, harrison2012bayesian} consider Bayesian and non-parametric approaches.

Closer to the theme of this paper, there has also been a considerable amount of research about dynamic pricing in models incorporating competition, \cite{Gallego2014, parlakturk2012value, gallego2014multiproduct} being some of them. However, most of these consider discrete choice models of demand, where a single consumer approaches and buys a discrete bundle of goods. Moreover, they assume that every seller has a fixed inventory level in the beginning and is not replenished during the course of the algorithm. We, on the other hand, consider demand originating from a general mass of consumers where when the volumes are large, the items may be considered divisible. For a more thorough survey of the existing literature, we refer the reader to \cite{Chen2015}.

In Section \ref{sec:warmup} we consider Online Gradient Descent (OGD), first introduced by Zinkevich \cite{Zinkevich} as the learning algorithm used by all sellers. At every time step, the learner takes a step in the direction of the gradient observed in that round. Interestingly, the author shows that this simple update rule achieves a regret bound of $O(\sqrt{T})$. While this approach is independent of any game-theoretic considerations Syrgkanis et al \cite{Syrgkanis} showed that with certain modified versions of this algorithm the individual regret of each player can be brought down to $O(T^{1/4})$. The analysis is based on the learning algorithm proposed by Rakhlin and Sridharan \cite{pmlr-v30-Rakhlin13} in a different context. Informally, the algorithm is based on the idea that if the gradient observed in the next round is \textit{predictable}, then it rules out the worst-case scenario and allows one to achieve a much better regret guarantee.

\section{Static Market Model}
\label{sec:static-market-model}

We consider a market with $n$ sellers, each selling a single good to a general population of consumers. We assume that the market operates in a round-based fashion. In each round $t$ every seller $i$ chooses a price $p_i^t$ for her good. The supply, $w_i$, of seller $i$, stays the same every round. No left-over supply from previous rounds is carried over (which is the case for example for perishable goods). Depending on the resulting price vector $\vec{p}^t = (p_i^t)_i$, each seller observes a certain demand for her item given by $x_i(\vec{p}^t)$. These observed demands are governed by an underlying utility function of the consumers. For the purposes of this paper (except Section~\ref{sec:warmup}), we assume that these utilities are ``\texttt{IGS}" as defined below: 

\begin{defn}[Iso-elastic and Gross Substitutes (\texttt{IGS}) utility]
\label{def:nice}
We say that a utility function is \texttt{IGS} when it satisfies the following conditions:
\begin{enumerate}[a)]
\item The utility function satisfies the gross substitutes property\footnote{Informally, this properties implies that increasing the price of a good $i$ does not decrease the demand of any other good $j$.} and the resulting demand functions are continuous.
\item Increasing the price of any good $i$ decreases the total spending on the item i.e. $p_i x_i(\vec{p})$.
\item The price elasticity of good $i$\footnote{Price elasticity is a measure of the percentage change in the quantity of a good demanded for a unit percentage change in the price i.e. $E_i(\vec{p}) = \left.  \frac{\partial x_j}{x_j} \middle/ \frac{\partial p_i}{p_i} \right.$   } for any price vector $\vec{p}$ satisfies:
\[
	\mod{\frac{\partial \ln x_j(\vec{p})}{\partial \ln p_i}} = E \hspace{1cm} \forall j \in [1,n]
\] where $E > 1$ is a constant.
\end{enumerate}
\end{defn}

We view this model of utilities as an approximation to the CES utilities (with the parameter $\rho \in (0, 1)$ ) used in several computer science and economics literature. It is a class of gross substitutes utility functions satisfying parts (a) and (b) in Definition~\ref{def:nice}. Instead of a fixed constant as price elasticity, this parameter depends on the prices of all goods i.e. $\mod{\frac{\partial \ln x_j(\vec{p})}{\partial \ln p_i}} = E_i(\vec{p})$. We use this more general class of utilities in Section~\ref{sec:warmup}.

In addition to the \texttt{IGS} utilities we make assumptions to ensure that the problem is well defined. Specifically, the optimal revenue of any seller $i$ for any profile $\vec{p}_{-i}$ of prices chosen by others is bounded in $[r, R]$. Intuitively, this is equivalent to saying that the set of allowed prices and supplies are such that revenue of any seller is not arbitrarily small or large.

We measure the performance of the pricing strategy used by the seller in terms of regret. Formally, the regret of an algorithm after $T$ rounds is defined as the loss with respect to the single best action (here price) in hindsight. For example, if $\{ r_i^t(p_i) \}_{t} $ denotes the sequence of revenue functions faced by the seller $i$ then the regret with respect to the sequence of prices $ \{p_i^t \}_{t=1}^T$ is defined as: $ R_T = \sum\limits_{t}  r_i^t(p_i^*) - r_i^t(p_i^t) $ where $p_i^* = \argmax\limits_{p}  \sum\limits_{t} r_i^t(p)$. Analogously, one can also define \textit{dynamic regret} as the regret incurred with respect to a dynamic benchmark sequence. For example, if $p_1^*, p_2^* \cdots p_T^*$ is the sequence of prices against which we measure the loss of our algorithm, then dynamic regret is defined as:
\[
	R_T(p_1^*, p_2^* \cdots p_T^*) = \sum\limits_{t}  r_i^t(p_t^*) - r_i^t(p_i^t)
\] 
\noindent
\textbf{Log-Revenue Objective:} In this paper, we take an indirect approach to the problem of revenue optimization by optimizing the log-revenue objective instead of the actual revenue. The log-revenue objective is simply the plot of revenue against the price in the log-scale defined as follows:
\[
	\ln r_i(\vec{p}) = \ln \left[ p_i \min\left\lbrace  x_i(\vec{p}), w_i \right\rbrace \right].
\] Using the definition of  \texttt{IGS} utility functions we can derive the following straightforward fact used directly in the rest of the paper. The proposition follows from the definition of log-revenue function and price elasticity of demand. 
\begin{proposition}
\label{prop:log_gradient}
The gradient of the log-revenue function $\tilde{r}_i(\tilde{p}_i)$ satisfies:
\begin{equation*}
\begin{aligned}
	\frac{\partial \tilde{r}_i}{\partial \tilde{p}_i} = 
	\begin{cases}
		1 - E &  \text{for } p_i: x_i(\vec{p}) < w_i \\
		1 &  \text{for } p_i: x_i(\vec{p}) \geq w_i 
\end{cases}
\end{aligned}
\end{equation*}
\end{proposition}
This proposition implies that the log-revenue function for seller $i$, keeping prices of all other items fixed, is shaped as in Figure~\ref{fig:nice-log-revenue}. It is instructive to keep this general shape in mind as we introduce learning algorithms to optimize it in the following sections.

\begin{figure}[!tbp]
  \centering
  \begin{minipage}[b]{0.45\textwidth}
    \includegraphics[width=\textwidth]{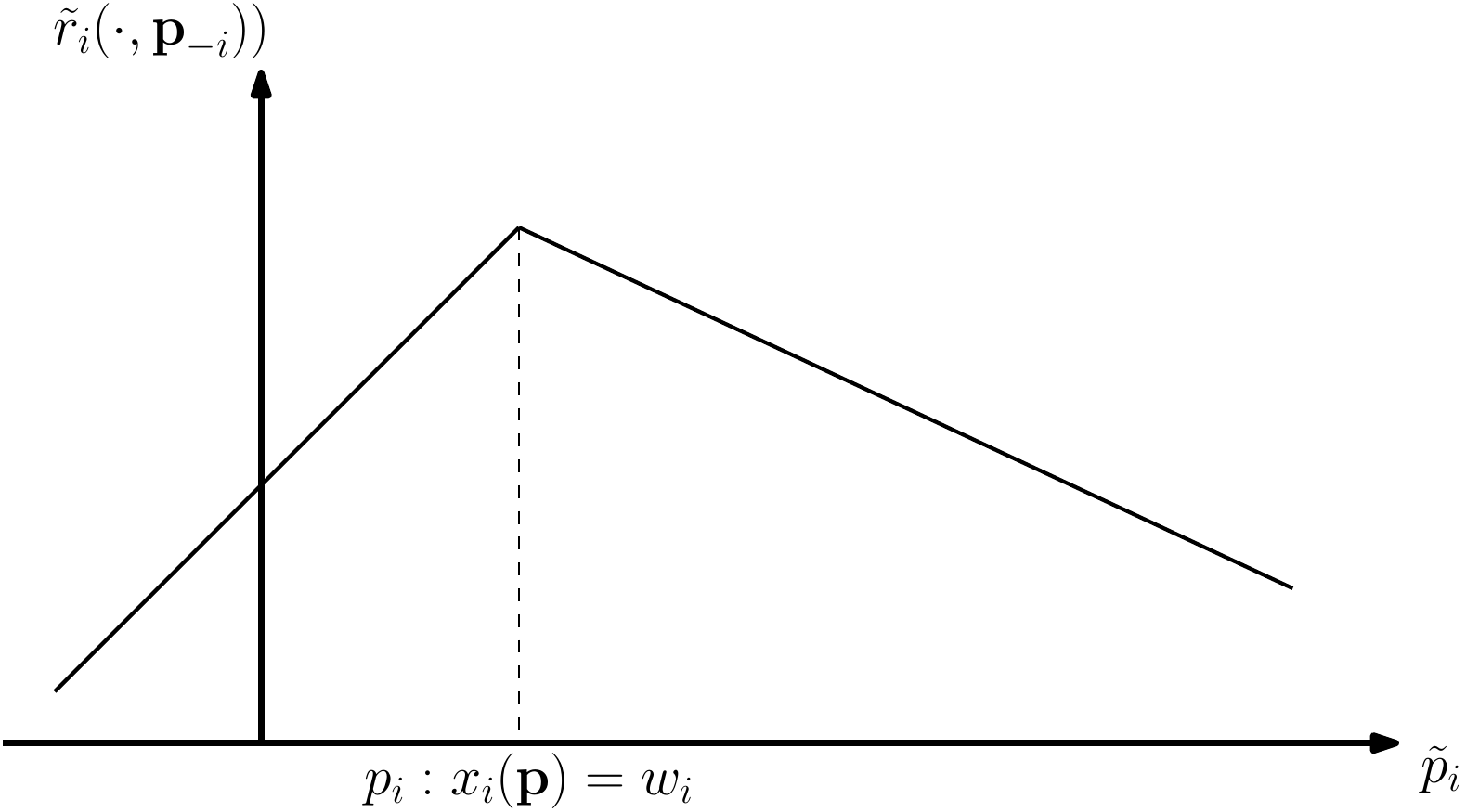}
    \caption{Log-revenue for \texttt{IGS} utilities}
     \label{fig:nice-log-revenue}
  \end{minipage}
  \hfill
  \begin{minipage}[b]{0.45\textwidth}
    \includegraphics[width=\textwidth]{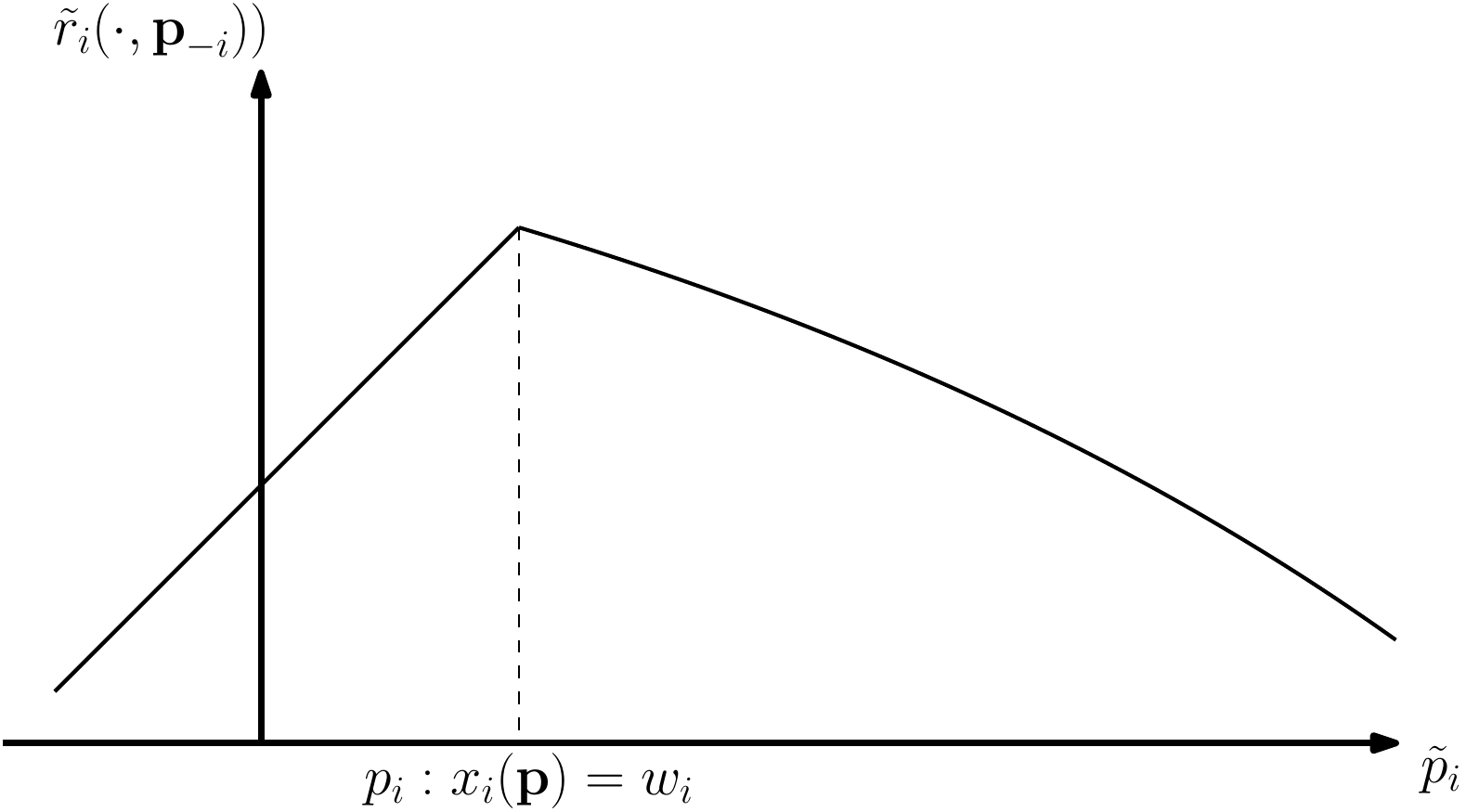}
    \caption{Log-revenue for CES utilities}
    \label{fig:ces-log-revenue}
  \end{minipage}
\end{figure}

\noindent
\textbf{Notation:} We shall denote vectors by bold-face letters and log of an entity by tilde, for example, $\ln r = \tilde{r}$. Often for ease of notation, we shall use $x_i$ to denote demand of good $i$ instead of $x_i(\vec{p})$ when it is clear from the context. $p_{-i}$ denotes the vector of prices of all sellers excluding $i$. The $\nabla$ notation denotes the gradient. All the missing proofs can be found in the Appendix.

\section{Modified OGD}
\label{sec:warmup}

In this section, we demonstrate the kind of regret bounds that can be achieved in full generality with CES utilities (with the parameter $\rho \in (0, 1)$). As noted before, since CES utilities do not satisfy the \texttt{IGS} utility model, the gradient of the log-revenue curve, $1 - E_i(\vec{p})$, in the case when $x_i(\vec{p}) < w_i$ is unknown to the seller (see in contrast Proposition ~\ref{prop:log_gradient}). To ensure that the problem is well-defined we assume that the price elasticity of demand for any item $i$ and any price vector $\vec{p}$  is bounded in $[ E_{min}, E_{max}]$. We work around the problem of unknown gradients by using a simple modification to the analysis by Zinkevich (Theorem 1, ~\cite{Zinkevich}) and show that if sellers use online gradient descent (with modified gradient feedback) as their learning algorithm on the log-revenue objective, then they can achieve a $O(\sqrt{T})$ regret bound.  We start with a claim for general convex functions with modified feedback.

\begin{myclaim}
\label{claim:OGD-extn}
Consider a sequence of convex functions $f_1, f_2 \cdots f_T$ satisfying the following condition:
\[
	g ~\leq~ | \nabla f_t(x)| ~\leq~ G  \hspace{1cm} \forall t \in [T], x \in \mathcal{X} 	
\]
\noindent
Suppose for the action $x_t$ chosen in round $t$ and for $\gamma = \frac{G}{g}$, we receive as feedback $\nabla g_t(x_t) \in \left[ \frac{\nabla f_t(x_t)}{\gamma}, \gamma\nabla f_t(x_t)  \right] $, then the regret bound of OGD  for step-size $\eta_t = 1 / \sqrt{t}$ is given by $R_T \leq \left( \gamma  \sqrt{T} \right)$.
\end{myclaim}

This property allows us to use OGD even with \textit{imperfect} gradient feedback, upto a multiplicative constant, to obtain regret bounds that are also within this same factor. Since the exact gradient in the case when $x_i(\vec{p}) < w_i$ is not available to the algorithm we modify the feedback gradient based on the demand observed,
\begin{equation}
\label{eqn:adjusted_gradient}
\begin{aligned}
	\frac{\partial \tilde{r}_i}{\partial \tilde{p}_i} = 
	\begin{cases}
		1 - E_i(p) & \wideRightarrow~ -1 ,  \hspace{1cm} \text{for } p_i: x_i(\vec{p}) < w_i \\
		1 &\wideRightarrow~ 1,   \hspace{1cm} \text{for } p_i: x_i(\vec{p}) \geq w_i
\end{cases}
\end{aligned}
\end{equation}
i.e. we work around this problem by choosing as feedback the gradient $-1$ whenever $x_i(\vec{p}) < w_i$ and $+1$ otherwise.

\begin{theorem}
\label{lem:OGD-extn}
	If any player $i$ uses OGD on the log revenue curve with $\eta_t = t^{-1/2}$ with the adjusted gradient feedback as in Equation~\ref{eqn:adjusted_gradient}, then the cumulative loss in revenue of seller $i$ is bounded as:
$$ \sum\limits_{t}  r_i^t(p_i^*) - r_i^t(p^t)  ~\leq~ O\left( R \cdot \max \left\lbrace E_{max}-1, \frac{1}{E_{min}-1} \right\rbrace  T^{1/2} \right),$$ 
where $ p_i^* = \argmax\limits_{p_i} \sum\limits_{t} \tilde{r}_i(p_i, p_{-i}^t)$.
\end{theorem}

\begin{proof}
Since the price elasticity of demand for any item $i$ at any price vector $\vec{p}$ satisfies $1 < E_{min} < \mod{E_i(\vec{p})} < E_{max}$, for the case when $x_i(\vec{p}) < w_i$, the gradient of log-revenue curve satisfies: $$ E_{min}-1 ~\leq~ \mod{\frac{\partial \tilde{r}_i}{\partial \tilde{p}_i}}  ~\leq~ E_{max}-1 .$$ Using the same idea as in Claim~\ref{claim:OGD-extn}, we can pretend to be using OGD on the actual log-revenue curve with a correspondingly modified step size $\eta' \in \left[ (E_{max}-1) \eta, \frac{\eta}{E_{min}-1} \right].$ The following bound then follows directly: $$\sum\limits_{t}~ \tilde{r}_i(p^*, p_{-i}^t) - \tilde{r}_i(p_i^t, p_{-i}^t) ~\leq~ O\left( \max \left\lbrace E_{max}-1, \frac{1}{E_{min}-1} \right\rbrace  T^{1/2} \right),$$  where $ p^* = \argmax\limits_{p_i} \sum\limits_{t} r_i(p_i, p_{-i}^t)$. The left-hand side of the above inequality can be further lower bounded: 
\begin{equation*}
\begin{aligned}
	\sum\limits_{t}~ \tilde{r}_i(p^*, p_{-i}^t) - \tilde{r}_i(p_i^t, p_{-i}^t)  & ~=~ -  \sum\limits_{t} \ln\left( 1 +  \frac{r_i^t(p_i^t) - r_i^t(p^*)}{r_i^t(p^*)}	\right) \\[5pt]
	& ~\geq~   \sum\limits_{t} \frac{r_i^t(p_i^*) - r_i^t(p^t)}{r_i^t(p^*)} ~\geq~   \sum\limits_{t} \frac{r_i^t(p_i^*) - r_i^t(p^t)}{R}
\end{aligned}
\end{equation*}
\qed
\end{proof}

This bound serves as a benchmark and improving upon this is the main focus of our paper. In the next section, we focus on a smaller set of \texttt{IGS} utility functions and show that with specialized learning algorithms the price dynamics converge faster to an approximately optimal configuration.

\section{Game Theoretic Interpretation}

\subsection{Preliminaries}


We start our investigation into this problem by observing that the revenue optimization problem in a market (as defined in Section \ref{sec:static-market-model}) is equivalent to agents in a game using learning algorithms locally to optimize their utility, where this utility is a function of the \textit{strategies} of all agents in the game. Problems of this flavour have already been studied in different game-theoretic settings but are not applicable in a black-box fashion to our problem on account of the market specific constraints. Specifically, the log-revenue objective although concave is not smooth, an assumption used in almost all gradient-based learning algorithms. This calls for a different approach than the ones taken in the idealized settings.

With this context in mind, we start from the result of~\cite{Syrgkanis}, where it is proved that if all players in a game use learning algorithms satisfying a certain technical property, called the \myprop property (See Definition~\ref{defn:RVU}), then the regret incurred by each individual agent is $O(T^{1/4})$. A natural question is then: Can we use the same technique in our revenue optimization problem in markets?

\begin{defn}[\myprop property, ~\cite{Syrgkanis}]
\label{defn:RVU}
We say that a vanishing regret algorithm satisfies the Regret bounded by Variation in Utilities (\myprop) property with parameters $\alpha > 0$ and $0 < \beta \leq \gamma$ and a pair of dual norms $ ( \norm{\cdot}, \norm{\cdot}_{*})$ if its regret on any sequence of utilities $\vec{u}^1, \vec{u}^2, \ldots \vec{u}^T$ is bounded as:
\[
    \sum\limits_{t=1}^T \dott{\vec{p}^* -  \vec{p}^t}{\vec{u}^t} ~\leq~ \alpha ~+~ \beta \sum\limits_{t=1}^T \norm{\vec{u}^t - \vec{u}^{t-1}}_{*} ~-~ \gamma \sum\limits_{t=1}^T \norm{\vec{p}^t - \vec{p}^{t-1}}
\]
\end{defn}

Although this property is defined for linear utility functions, we can extend this definition to concave utilities by using the gradient of the utility with respect to $p_i$ as proxy for $\vec{u}^t$ i.e. in the context of our problem $$ \tilde{r}_i^t(p_i^*) - \tilde{r}_i^t(p_i^t)  ~\leq~ \dott{\vec{p}^* -  \vec{p}^t}{\frac{\partial \tilde{r}_i}{\partial \tilde{p}_i}} . $$  As noted in~\cite{Syrgkanis}, the standard online learning algorithms such as Online Mirror Descent (generalization of OGD) and Follow-the-Regularized-Leader (FTRL) do not satisfy the \myprop property. However, Rakhlin and Sridharan~\cite{pmlr-v30-Rakhlin13} and Syrgkanis~\cite{Syrgkanis} et al have developed modified versions of these algorithms, namely Optimistic Mirror Descent (OMD) and Optimistic FTRL (OFTRL) respectively, that do satisfy this property, 

\begin{proposition}[Informal, \cite{Syrgkanis}]
Let $D$ denote a measure of the diameter of the decision space. Then:
\begin{enumerate}
\item The OMD algorithm using step size $\eta$ satisfies the \myprop property with constants $\alpha = D /\eta$, $\beta = \eta$ and $\gamma = 1/(8 \eta)$ 

\item The OFTRL algorithm using step size $\eta$ satisfies the \myprop property with constants $\alpha = D /\eta$, $\beta = \eta$ and $\gamma = 1/(4 \eta)$ 
\end{enumerate}
\end{proposition}


In the context of continuous games, the utility function (alternatively, the objective) of each player should additionally satisfy some \textit{regularity} conditions. For ease of presentation, we shall refer to the player objectives satisfying these conditions as \textit{regular objectives} and are defined, in a general sense, as follows:

\begin{defn}[Regular Objective]
\label{defn:regular}
Let the strategy space of each player $i$ be denoted by $S_i \in \mathbb{R}^d$ and the combined strategy space by $\mathcal{S} = S_1 \times S_2 \times \cdots S_n$.
Let $\vec{w} = (\vec{w}_i )_{i=1}^n$ denote the combined strategy profile where the strategy of each player $w_i \in S_i$. An objective function $f_i: \mathcal{S} \rightarrow \mathbb{R}$ of a player $i$ is said to be regular if it satisfies the following conditions:
\begin{enumerate}
\item (Concave in player strategy) For each player $i$ and for each profile of opponent strategies $\vec{w}_{-i}$, the function $f_i(\cdot, \vec{w}_{-i})$ is concave in $\vec{w}_{i}$.
\item (Lipschitz Gradient) For each player $i$, the gradient of the objective with respect to $i$, $\delta_i(\vec{w}) = \nabla_i f_i(\vec{w})$ is $L$-Lipschitz continuous with respect to the L1-norm. i.e.
\[
    \norm{\delta_i(\vec{w}) - \delta_i( \vec{y}) }_{*} ~\leq~ L \cdot \norm{ \vec{w} - \vec{y}}.
\]
\end{enumerate}
\end{defn}

\subsection{Smoothed Log-Revenue Curve}

One of the foremost requirements to apply the analysis based on the \myprop property is that the utility function should be smooth, specifically, the gradient of the objective should be $L$-Lipschitz continuous.\footnote{Informally, this is required to ensure that small changes in prices do not lead to large changes in utility gradient.} Clearly, as seen in Figure~\ref{fig:nice-log-revenue}, this is not the case with our log-revenue objective. We work around this problem by using a \textit{smoothed} gradient feedback.

\begin{defn}[Smoothed Gradient Feedback]
\label{defn:smoothed-gradient}
    For any fixed seller $i$ and price vector $\vec{p}_{-i}$, we define the smoothed gradient for player $i$, $\delta_{i, \mathsf{X}_i}(\cdot)$, as follows:
    \[
        \delta_{i, \mathsf{X}_i}(p_i) = 
    \begin{cases}
        1 ,  & \text{for } p_i: x_i(\vec{p}) > w_i \\
        1 - E,  & \text{for } p_i: x_i(\vec{p}) <  \mathsf{X}_i \\[5pt]
        1 + \frac{E ( \tilde{x}_i(\vec{p}) - \tilde{w}_i)}{\tilde{w}_i - \tilde{\mathsf{X}}_i}, & \text{otherwise }
\end{cases}
    \]
    where $\mathsf{X}_i$ is a threshold parameter for seller $i$.
\end{defn}

For ease of notation, we shall denote $\delta_{i, \mathsf{X}_i}(p_i)$ by simply $\delta_{i}$ when clear from context. For purposes of analysis, we parametrize the threshold parameter of seller $i$ as $\mathsf{X}_i = \frac{w_i}{\exp\left( \epsilon r\right)}$ where $\epsilon$ is a small constant and $r$ is a lower bound on optimal revenue of seller $i$. Also, henceforth we shall refer to the \textit{actual} revenue curve by $\tilde{r}(\cdot)$ and the algorithm's view of smoothed revenue curve by $\tilde{r}^{sm}(\cdot)$ .

\begin{figure}[!tbp]
  \centering
  \begin{minipage}[b]{0.45\textwidth}
    \includegraphics[width=\textwidth]{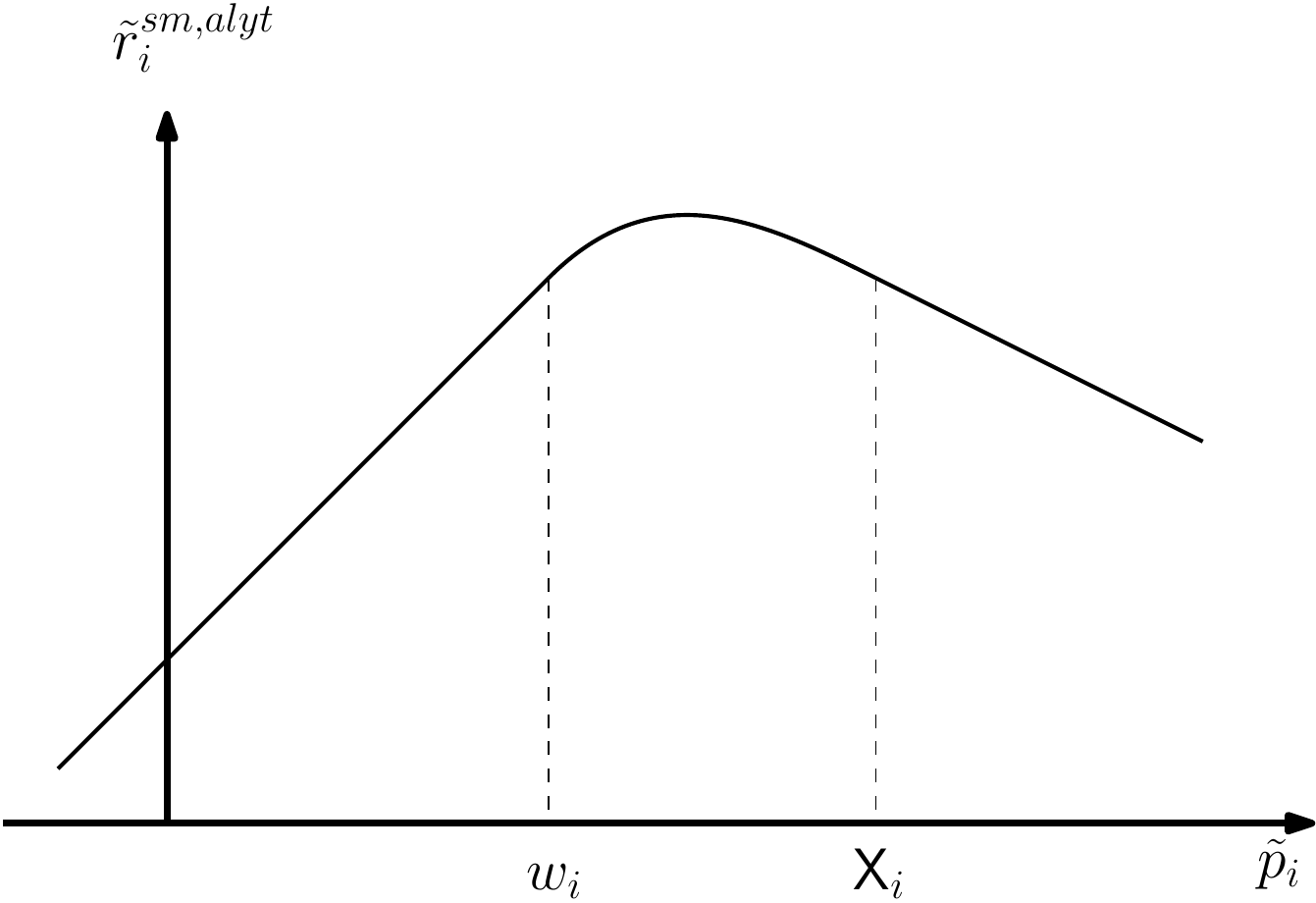}
    \caption{Smoothed log-revenue from an analytical standpoint}
     \label{fig:smoothed-log-revenue}
  \end{minipage}
  \hfill
  \begin{minipage}[b]{0.45\textwidth}
    \includegraphics[width=\textwidth]{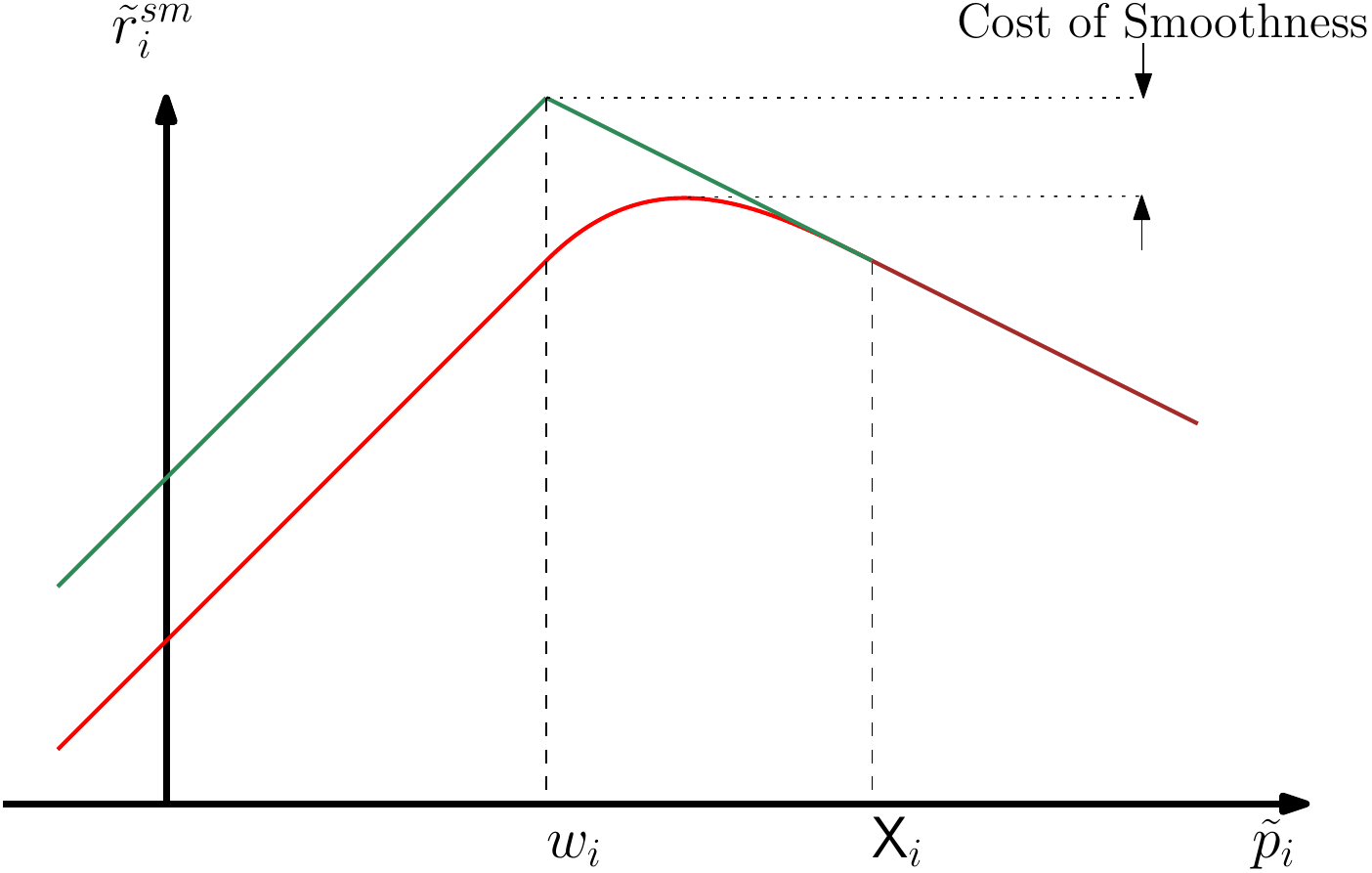}
    \caption{Smoothed vs actual log-revenue curve}
    \label{fig:smoothed-vs-actual-log-revenue}
  \end{minipage}
\end{figure}

\begin{lemma}
The smoothed revenue objective, $ \tilde{r}_i^{sm}(\vec{p}) $, for any seller $i$ is regular.
\end{lemma}
 

\subsection{Cost of Smoothness}

Since our learning algorithm only uses the smoothed gradient feedback the resulting regret bound also holds only for the smoothed view of the log-revenue curve, i.e. the optimal price in this smoothed view would be the price for which the smoothed gradient is zero although this price is clearly suboptimal for the actual revenue curve. (See Figure~\ref{fig:smoothed-vs-actual-log-revenue}). To prove bounds with respect to the actual revenue curve we need to draw connections between the smoothed and actual revenue for any fixed price.

\begin{lemma}
\label{lem:max-diff}
For any seller $i$ and fixed $\vec{p}_{-i}$ and for any fixed price $p$ chosen by seller $i$:
\[
    0 ~\leq~ \tilde{r}_i(p, \vec{p}_{-i}) ~-~ \tilde{r}^{sm}_i(p, \vec{p}_{-i}) ~\leq~   \epsilon r
\]
\end{lemma}

\begin{theorem}
\label{thm:main1}
Suppose each seller $i$ uses the OFTRL algorithm on the log-revenue objective using the smoothed gradient feedback and threshold demand $\mathsf{X}_i ~=~ \frac{w_i}{\exp\left( \epsilon r\right)}$.  Let $ p_i^{**} = \argmax_{p}~ \sum_{t} \tilde{r}_i^t(p) $ denote the optimal price in hindsight with respect to the log-revenue objective. Then  the actual loss in revenue is bounded as: $$  \sum\limits_{t=1}^{T} \left( 1 - \epsilon R \right) r_i^t(p_i^{**}) - r_i^t(p^t)  ~\leq~  O\left( \left(  \frac{R^2 E^2}{\epsilon r} \right)^{1/2} T^{1/4} \right) ~-~  \epsilon R T. $$
\end{theorem}

\begin{proof}
Since $\tilde{r}_i^{sm}(p_i, \vec{p}_{-i})$ satisfies the regularity condition (Definition~\ref{defn:regular}), if each seller uses a learning algorithm satisfying the \myprop property, then the individual regret satisfies:
\begin{equation*}
\begin{aligned}
\sum\limits_{t} \tilde{r}_i^{sm}( p_i^{**}, \vec{p}^t_{-i}) ~-~ \tilde{r}_i^{sm}(p_i^t, \vec{p}^t_{-i}) ~\leq~ & \sum\limits_{t} \tilde{r}_i^{sm}(\bar{p}_i^*, \vec{p}^t_{-i}) ~-~ \tilde{r}_i^{sm}(p_i^t, \vec{p}^t_{-i})\\
     ~\leq~ & \sum\limits_{t}     \dott{\delta_{i, \mathsf{X}_i}(\vec{p}^t)}{\tilde{ \bar{p}}_i^* - \tilde{p}_i^t}
\end{aligned}
\end{equation*}
where $ \bar{p}_i^* = \argmax\limits_{p} \sum\limits_{t} \tilde{r}_i^{sm}(p, \vec{p}^t_{-i}) $. For ease of notation, we denote $\delta_{i, \mathsf{X}_i}(\vec{p}^t)$ by $\dit$. Using Lemma~\ref{lem:max-diff} to lower bound the left-hand-side above:
\begin{equation}
\label{eqn:main-proof-lhs}
\begin{aligned}
    \sum\limits_{t} \tilde{r}_i^{sm}( p^{**}, \vec{p}^t_{-i}) ~-~ \tilde{r}_i^{sm}(p_i^t, \vec{p}^t_{-i}) \geq & \enspace  \sum\limits_{t} \left(  \tilde{r}_i^t(p_{i}^{**}) ~-~  \epsilon r \right)  ~-~  \tilde{r}_i(p_{i}^t)\\
    \geq & \enspace  \sum\limits_{t} \left(  1 - \epsilon \right)  \tilde{r}_i(p_{i}^{**}) ~-~ \tilde{r}_i(p_{i}^t)
\end{aligned}
\end{equation}
The last inequality holds since $r$ is the lower bound on revenue. We still have to prove an upper bound on the expression: $ \sum\limits_{t}     \dott{\dit}{\tilde{ \bar{p}}_i^* - \tilde{p}_i^t}$. Since our learning algorithm satisfies the \myprop property, by Definition~\ref{defn:RVU} it follows that: 
\begin{equation*}
    R_T \leq \alpha ~+~ \beta \sum\limits_{t=1}^T \mod{\dit - \dito}^2 .
\end{equation*}
Since the smoothed gradient $\delta_i(\vec{p})$ for any seller is $L$-Lipschitz continuous (Lemma~\ref{lem:Lipschitz}), for $L = \frac{E^2}{\epsilon r}$ we can bound $\mod{\dit - \dito}^2 $ as:
\begin{equation*}
\begin{aligned}
    \mod{\dit - \dito}^2 \leq & \enspace L^2 \left( \sum\limits_j \mod{p_j^t - p_j^{t-1}} \right)^2\\[5pt]
    \leq & \enspace L^2 n \sum\limits_j \mod{p_j^t - p_j^{t-1}}^2 .
\end{aligned}
\end{equation*}

In addition to the fact that OFTRL satisfies the \myprop property, it is also known that the algorithm satisfies a \textit{stability} property (Lemma 20,~\cite{Syrgkanis}) i.e.  $\mod{p_j^t - p_j^{t-1}} \leq 2 \eta $ where $\eta$ is the step-size parameter of the algorithm.

We can now bound the regret as: $ R_T \leq \alpha ~+~ 4 n^2 \beta L^2 \eta^2 T $. Finally substituting the \myprop parameters of the algorithm  (Proposition 7, \cite{Syrgkanis}) $\alpha = D / \eta$, $\beta = \eta$ and $\gamma = 1 /4 \eta$ with $\eta  = (Ln)^{-1/2} T^{-1/4}$ we get: $$R_T \leq D / \eta + 4 \eta^3 L^2 n^2 T = O( \sqrt{Ln} (D+4) T^{1/4}).$$ Combining this with Equation~\ref{eqn:main-proof-lhs} and substituting the value of $L$ we get:
\begin{equation*}
\label{eqn:main-proof-part1}
    \sum\limits_{t=1}^T  \left( 1 - \epsilon \right) \tilde{r}_i^t(p_i^{**}) - \tilde{r}_i^t(p_i^t) ~\leq~ O\left( \left(  \frac{E^2}{\epsilon r} \right)^{1/2} \cdot T^{1/4} \right)
\end{equation*}
\noindent
Rearranging the inequality and using same steps as in the proof of Lemma~\ref{lem:OGD-extn}:

\begin{align*}
      \sum\limits_{t}  \frac{r_i^t(p_i^{**}) - r_i^t(p^t)}{r_i^t(p^{**})}  ~\leq~&  O\left( \left(  \frac{E^2}{\epsilon r} \right)^{1/2} \cdot T^{1/4} \right) +  \epsilon \sum\limits_{t=1}^T \tilde{r}_i^t(p_i^{**}) \MoveEqLeft[1]\\
      \sum\limits_{t} r_i^t(p_i^{**}) - r_i^t(p^t)  ~\leq~ &  O\left( \left(  \frac{E^2 R^2}{\epsilon r} \right)^{1/2} \cdot T^{1/4} \right) +  \epsilon R \sum\limits_{t=1}^T \tilde{r}_i^t(p_i^{**})\\
      ~\leq~ &  O\left( \left(  \frac{R^2 E^2}{\epsilon r} \right)^{1/2} T^{1/4} \right) +  R \epsilon \sum\limits_{t=1}^T ( r_i^t(p_i^{**}) - 1)
\end{align*}

\begin{align*}
        \sum\limits_{t} \left( 1 - \epsilon R \right) r_i^t(p_i^{**}) - r_i^t(p^t)  ~\leq~ &  O\left( \left(  \frac{R^2 E^2}{\epsilon r} \right)^{1/2} T^{1/4} \right)  -  \epsilon R T
\end{align*}
\qed
\end{proof}
Similar bounds can be shown in the case when sellers use the Optimistic Mirror Descent (OMD) algorithm.

\begin{remark}
Here we compare the total revenue obtained to the total revenue with respect to the fixed price $ p^{**} = \argmax\limits_{p}~ \sum\limits_{t} \tilde{r}_i^t(p) $ i.e. the price in hindsight that optimizes the cumulative log-revenue objective and not necessarily the revenue objective itself. We note that since the revenue function need not be concave, it is not immediately clear how to characterize the resulting cumulative revenue function and the price optimizing it. For this reason, we are using the price that optimizes the cumulative log-revenue.
\end{remark}

\section{Learning with a Dynamic Benchmark}
A bound on the loss of revenue of a seller with respect to the single price $p_i^{**}$ in hindsight is a comparatively weak benchmark. Ideally the sellers would like to choose as benchmark the revenue-optimizing price in every round, i.e. the sequence of prices $ \{ p_i^{*,t} \}_{t=1}^T$. Such a benchmark is however too strict to obtain meaningful regret bounds. We shall instead focus on a more constrained sequence of benchmark prices. In what follows, we define a class of learning algorithms whose guarantees apply to \textit{any} game setting where strategic players use regret minimization to maximize their own utility.  For generality, we define this class for any sequence of concave utility functions $\{ u_i^t(\cdot) \}_{t}$. In the following section, we shall specialize this guarantee to the context of revenue optimization in markets.

\begin{defn}[DRVU property]
\label{defn:drvu_property}

We say that a vanishing regret algorithm satisfies the Dynamic Regret bounded by Variation in Utilities (DRVU) property with parameters $\alpha, \rho > 0$ and $0 < \beta \leq \gamma$ and a pair of dual norms $ ( \norm{\cdot}, \norm{\cdot}_{*})$, if its regret on any sequence of utilities $\vec{u}^1, \vec{u}^2, \ldots \vec{u}^T$ with respect to the benchmark sequence $\{ p_i^{*, t} \}_{t}$ is bounded as:

\begin{align*}
	\sum\limits_{t=1}^T \dott{\vec{p}^{*,t} -  \vec{p}^t}{\vec{u}^t} ~\leq~ & \alpha  ~+~ \beta \sum\limits_{t=1}^T \norm{\vec{u}^t - \vec{u}^{t-1}}_{*}^2\\
	&  + ~ \rho ~  \sum\limits_{t=1}^T   \norm{\vec{p}^{*, t} -~ \vec{p}^{*, t-1}} ~-~ \gamma \sum\limits_{t=1}^T \norm{\vec{p}^t - \vec{p}^{t-1}}.
\end{align*}
	
\end{defn}
\vspace{8pt}

This definition is an extension of the \myprop property. The difference is in the term $\rho ~ \sum\limits_{t}  ~ \norm{\vec{p}^{*, t} -~ \vec{p}^{*, t-1}}$ that quantifies the hardness of learning with respect to a dynamic strategy. As for the \myprop property, this property is defined with respect to linear utilities and can be extended to concave utilities by standard arguments.

\begin{lemma}[Informal]
The OMD algorithm, with step size $\eta$ and suitably chosen parameters, satisfies the DRVU property with constants $\alpha = D_1 /\eta$, $\rho = D_2/ \eta $ $\beta = \eta$ and $\gamma = 1/(8 \eta)$ for constants $D_1$ and $D_2$.
\end{lemma}

Using this new definition we can now extend almost all of the results in~\cite{Syrgkanis} to corresponding results for dynamic regret. We state the following claim for concreteness.

\begin{cor}
\label{corr:DRVU-regret}
Let $C_T = \sum_{t}   \parallel p_i^{*, t} - p_i^{*, t-1}\parallel$ denote the cumulative change in benchmark strategies of player $i$. If all players use algorithms satisfying the DRVU property, then the regret incurred by any player $i$ satisfies:
\[
	\sum\limits_{t} u_i^t \left( p_i^t, p_{-i}^t \right)  -~  u_i^t \left(p_i^{*, t}, p_{-i}^t \right) ~\leq~ O\left( (1 + C_T) T^{1/4} \right)
\]
\end{cor} 

\subsection{Revenue Optimization in Dynamic Markets}


\textbf{Dynamic Market Model:} We define a dynamic market $\mathcal{M} = (M_1, M_2 \cdots M_T)$, as a sequence of markets with the same set of sellers and buyers, with the same \textit{nice} utility functions as in Definition~\ref{def:nice} but with a dynamic supply vector i.e. we characterize the dynamicity of the market by the sequence of supply vectors $\vec{w}_1, \vec{w}_2 \cdots \vec{w}_T$. In order to achieve a strong dynamic regret bound, we shall assume that the income elasticity parameter of the market is equal to one. This is a standard assumption in many market models and is also satisfied by CES utilities.

In this section, we connect the dynamic regret of any seller $i$ to the inherent instability of the market by choosing the sequence of \textit{equilibrium prices}\footnote{Informally, a (Walrasian) equilibrium in this market corresponds to the vector of prices and an allocation of items such that no item is under- or over-demanded. Alternatively, the aggregate demand for each item is exactly equal to its supply.} for seller $i$ at each round as the benchmark sequence, i.e. $\{ p_i^{eq, t} \}_{t=1}^T$. Since the supply vector may change every round, the equilibrium prices may also correspondingly change. These changes in equilibrium prices completely capture the inherent instability of the market. For example, if the supply stays the same every round, then this benchmark is the same as choosing the equilibrium price in each round. On the other hand, if the supply fluctuates wildly from one round to the next, then so do the equilibrium prices and there is no hope of achieving a sub-linear regret bound. That is, the resulting dynamic regret bound captures the inherent market instability. In our following theorem, we use this connection to prove a bound on dynamic regret with respect to the cumulative change in the supplies.
%
%
%
%
%

\begin{theorem}
\label{thm:dyn-revenue-loss}
Let $W_T = \sum_{t} \norm{\tilde{\vec{w}}^t - \tilde{\vec{w}}^{t-1}}_{1}$ denote the cumulative change in the market in terms of changes in supplies. Suppose each seller $i$ uses the OMD algorithm on the log-revenue function with smoothed gradient feedback and threshold demand $\mathsf{X}_i^t ~=~ \frac{w_i^t}{\exp\left( \epsilon r\right)}$.  Let $\{ p_i^{eq,t} \}_{t}$ denote the sequence of equilibrium prices for seller $i$. Then:
\[
 \sum\limits_{t=1}^T  \left( 1 - \epsilon R \right) r_i^t(p_i^{eq,t}) - r_i^t(p^t) ~\leq~ O\left( \left(  \frac{R^2 E^2}{\epsilon r} \right)^{1/2} \cdot (1 + W_T) T^{1/4} \right)
\]
\end{theorem}

\section{Experimental Evaluation}

We analyze the performance of our modified OGD and OMD algorithms when the consumer utility functions satisfy the CES property. Although from a theoretical standpoint we assumed that the price elasticity of the market is a constant, empirically we observed that CES functions approximately satisfy this assumption. In our simulations, we show that the OMD algorithm indeed performs as proved in our analysis, except for slightly worse convergence time.

\begin{wrapfigure}{l}{0.52\textwidth}
     \includegraphics[width=0.51\textwidth]{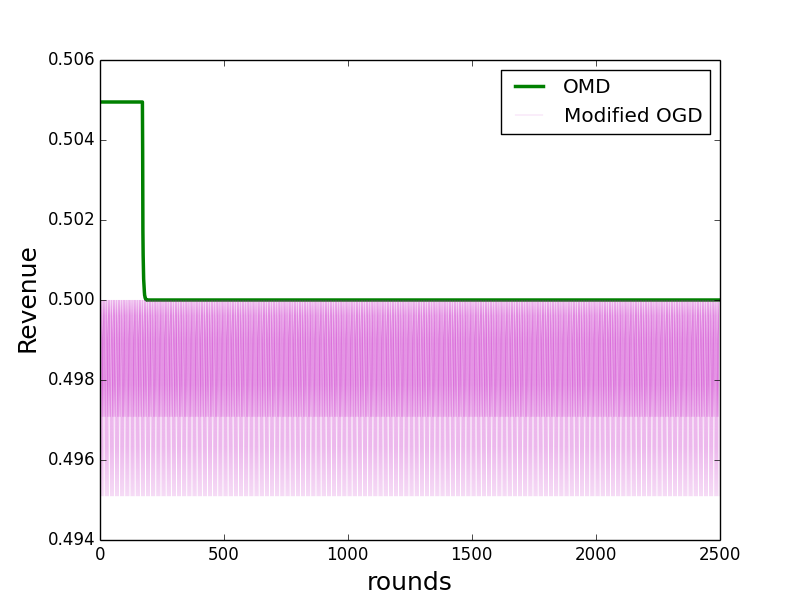}
  \caption{Modified OGD vs OMD}
\end{wrapfigure}

We consider the scenario with 2 items and the value of $E = 2.5$. We assume that the market is static in that each seller has a supply of one unit every round and uses the threshold parameter $\mathsf{X}_i = 0.9$. We observe that the modified OGD algorithm converges quickly to the neighbourhood of the optimal price but then keeps oscillating around it. This is expected since in this neighbourhood the observed gradients might change abruptly. The OMD algorithm on the other hand takes a while before it comes close to the neighbourhood but once there converges to optimum quickly. As described in the analysis, this is precisely the reason for using the smoothed gradient feedback.

\section{Conclusion}

In this paper, we presented two dynamic pricing strategies based on regret-minimizing algorithms for static markets. In contrast to a simple approach based on the modified OGD algorithm we showed that by using specialized learning algorithms the sellers can converge to (approximate) revenue maximizing prices. We extended the analysis of these algorithms to dynamic markets and proved corresponding dynamic regret bounds. In the process, we defined a property analogous to the RVU property that is satisfied by these learning algorithms and extended their results to the case of dynamic regret.

Our regret analysis with these specialized learning algorithms depends on the assumption that the underlying market is iso-elastic. We believe that extending the analysis to cases where the price elasticity may be dynamic is an important open question. Also, to obtain a regret bound in dynamic markets we needed the assumption of \textit{gross substitutes} utility function. Obtaining revenue guarantees for more general utility functions would be an interesting future direction.

\newpage
\begin{small}
\bibliographystyle{plain}
\bibliography{bib-rev-opt}
\end{small}

\newpage
\appendix
\begin{center}
 {\bf \Large APPENDIX}
  \end{center}
\section{Modified OGD}

\begin{myclaim}
Consider a sequence of convex functions $f_1, f_2 \cdots f_T$ satisfying the following condition:
\[
	g ~\leq~ | \nabla f_t(x)| ~\leq~ G  \hspace{1cm} \forall t \in [T], x \in \mathcal{X} 	
\]
\noindent
Suppose for the action $x_t$ chosen in round $t$ and for $\gamma = \frac{G}{g}$, we receive as feedback $\nabla g_t(x_t) \in \left[ \frac{\nabla f_t(x_t)}{\gamma}, \gamma\nabla f_t(x_t)  \right] $, then the regret bound of OGD  for step-size $\eta_t = 1 / \sqrt{t}$ is given by $R_T \leq \left( \gamma  \sqrt{T} \right)$.
\end{myclaim}

\begin{proof}
The update rule of OGD algorithm when the feedback, $\nabla f_t(x_t)$, is available is given by: $\vec{x}_{t+1} = \Pi \left[ \vec{x}_{t} - \eta_t \cdot \nabla f_t(\vec{x}_t) \right]$ where $\Pi(\cdot)$ is the euclidean projection operator. Since we use $\nabla g_t(x_t)$ instead, we would get a different sequence of decision points according to the update step as follows:
\[	
	\vec{x}'_{t+1} = \Pi \left[ \vec{x}'_{t} - \eta_t \cdot \nabla g_t(\vec{x}'_t) \right]
\]
Since $\nabla g_t(\vec{x}'_t) \in \left[ \frac{\nabla f_t}{\gamma}, \gamma\nabla f_t  \right] $, we can re-write the same update step as:
\[	
	\vec{x}'_{t+1} = \Pi \left[ \vec{x}'_{t} - \eta'_t \cdot \nabla f_t(\vec{x}'_t) \right]
\] where $\eta_t' \in \left[   \frac{\eta_t}{\gamma}, \gamma \eta_t  \right]$ is such that $ \eta_t \cdot \nabla g_t(\vec{x}_t) = \eta_t' \cdot \nabla f_t(\vec{x}_t)$, i.e. we get the same sequence of steps by using $\nabla f_t$ but with difference step size sequence. Following the same analysis as in Zinkevich~\cite{Zinkevich} and replacing $\eta_t$ by $\eta'_t$, the claim follows.
\qed
\end{proof}

\section{Regularity of Smoothed Revenue Objective}

\begin{lemma}
\label{lem:Lipschitz}
For any seller $i$, the gradient of the smoothed revenue curve with threshold demand $\mathsf{X}_i ~=~ \frac{w_i}{\exp\left( \epsilon r\right)}$ satisfies $\frac{ E^2}{\epsilon r}$-Lipschitz continuity i.e.
\begin{equation}
\label{eqn:lemma}
	\norm{\delta_i(\vec{p}^1) - \delta_i( \vec{p}^2) }_{*} ~\leq~ \frac{E^2}{\epsilon r} \cdot \norm{\tilde{p}^1 - \tilde{p}^2}.
\end{equation}
\end{lemma}

\begin{proof}
It is known that (Lemma 24,~\cite{Syrgkanis}) if for all $j$, $$ \norm{\delta_i(\vec{p}^1) - \delta_i( p_j^2, \vec{p}_{-j}^1) }_{*} \leq \frac{E^2}{\epsilon r} \cdot \norm{\tilde{p}_j^1 - \tilde{p}_j^2} $$ then $\delta_i(\cdot)$ satisfies inequality~\ref{eqn:lemma}. We shall first prove the case when $j$ is equal to $i$.

This is equivalent to proving $ \frac{\partial \delta_i(\vec{p}^1)}{\partial \tilde{p}_i} \leq \frac{E^2}{\epsilon r}$ since the revenue curve is differentiable. By observation, we note that the maximum change in smoothed gradient i.e. $\delta_i(\vec{p})$ occurs for prices when $ \mathsf{X}_i \leq x_i(\vec{p}) \leq w_i$. This implies:
\begin{equation*}
\begin{aligned}
	\mod{\frac{\partial \delta_i(\vec{p}^1)}{\partial \tilde{p}_i}} \leq & \enspace  \mod{\frac{\partial }{\partial \tilde{p}_i} \left( 1 + \frac{E (\tilde{x}_i(\vec{p}) - \tilde{w}_i)}{\tilde{w}_i - \tilde{\mathsf{X}}_i}\right) } \\[6pt]
	= & \enspace \mod{ \frac{E}{\tilde{w}_i - \tilde{\mathsf{X}}_i } \cdot \frac{\partial }{\partial \tilde{p}_i} \tilde{x}_i(\vec{p}) }\\[6pt]
	= & \enspace \mod{ \frac{E}{\epsilon r} \cdot - E_i(\vec{p}) }\\[6pt]
	\mod{\frac{\partial \delta_i(\vec{p}^1)}{\partial \tilde{p}_i} } \leq & \enspace \frac{E^2}{\epsilon r}
\end{aligned}
\end{equation*}
In a similar way, we can show that the smoothed gradient of seller $i$ is Lipschitz continuous also with respect to the price of any other seller $j$, i.e. $ \frac{\partial \delta_i(\vec{p}^1)}{\partial \tilde{p}_j} \leq \frac{E^2}{\epsilon r}$. Using the same arguments as above, we get: $$ \frac{\partial \delta_i(\vec{p}^1)}{\partial \tilde{p}_j} \leq  \frac{E}{\tilde{w}_i - \tilde{\mathsf{X}}_i } \cdot \mod{ \frac{\partial }{\partial \tilde{p}_j} \tilde{x}_i(\vec{p})}.$$ The cross derivative term $\frac{\partial \tilde{x}_i(\vec{p})}{\partial \tilde{p}_j}$ is exactly the \textit{cross-price elasticity} of item $i$ with respect to item $j$. We shall denote it by $E_{ij}(\vec{p})$ and by definition of \texttt{IGS} utility functions is exactly $E$. Therefore,  $$ \mod{\frac{\partial \delta_i(\vec{p}^1)}{\partial \tilde{p}_j} } \leq \enspace \frac{E^2}{\epsilon r}. $$
\qed
\end{proof}

\begin{cor}
The smoothed revenue objective, $ \tilde{r}_i^{sm}(\vec{p}) $, for any seller $i$ is regular.
\end{cor}
 The corollary follows from the fact that $\tilde{r}_i^{sm}(\vec{p})$ is concave in $\tilde{p}_i$ and Lemma~\ref{lem:Lipschitz}.  We will need this fact in the next section to obtain a regret bound using smoothed gradient feedback.

\section{Cost of Smoothness}

\begin{lemma}
\label{lem:rev-bound}
For a fixed $\vec{p}_{-i}$ let $p_{i, \mathsf{X}_i}$ denote the price such that $x_i(p_{i, \mathsf{X}_i}, \vec{p}_{-i}) = \mathsf{X}_i$, where $\mathsf{X}_i = \frac{w_i}{\exp\left( \epsilon r \right)}$ is the threshold demand of seller $i$. Then:
\[
	 \tilde{r}_i(p_{i}^*, \vec{p}_{-i}) ~-~ \tilde{r}_i(p_{i, \mathsf{X}_i}, \vec{p}_{-i})  ~=~ \frac{E - 1}{E} \cdot \epsilon r
\]
\end{lemma}
\begin{proof}
The lemma follows directly from the following two observations:
\begin{enumerate}
\item For any price $p_i > p_i^*$, where  $p_i^*$ is the revenue maximizing price of seller $i$, chosen by seller, $$\tilde{r}_i(p_i^*) - \tilde{r}_i(p_i) = (E - 1) (\tilde{p}_i - \tilde{p}_i^* ). $$ This follows from our assumption that the gradient of log-revenue curve for any price $p_i > p_i^*$ is a constant equal to $-(E-1)$.
\item For $\tilde{p}_{i, \mathsf{X}_i}$ and $\tilde{p}_i^*$ as defined above, the following holds: $$\tilde{p}_{i, \mathsf{X}_i} - \tilde{p}_i^* ~=~ \frac{\epsilon r}{E}.$$ This can be shown by the following sequence of utilities.

\begin{equation*}
\begin{aligned}
	\tilde{r}_i(p_i^*) & ~=~ \enspace \tilde{r}_i(p_{i, \mathsf{X}_i}) + (E -1) (\tilde{p}_{i, \mathsf{X}_i} - \tilde{p}_i^*)\\[6pt]
	\tilde{p}_i^* + \tilde{w}_i & ~=~  \enspace\tilde{p}_{i, \mathsf{X}_i} + \tilde{x}_i(p_{i, \mathsf{X}_i}) + (E-1)  (\tilde{p}_{i, \mathsf{X}_i} - \tilde{p}_i^*)\\[7pt]
	 \tilde{w}_i  -  \tilde{x}_i(p_{i, \mathsf{X}_i}) & ~=~  \enspace E (\tilde{p}_{i, \mathsf{X}_i} - \tilde{p}_i^*)
\end{aligned}
\end{equation*}
\noindent
Since $ \tilde{x}_i(p_{i, \mathsf{X}_i}) = \tilde{\mathsf{X}}_i$, it follows that $$\tilde{w}_i  -  \tilde{x}_i(p_{i, \mathsf{X}_i}) ~= \enspace \ln \left( \frac{w_i}{x_i(p_{i, \mathsf{X}_i})}\right) ~=~ \epsilon r.$$ The claim follows by using this in the above equality.
\end{enumerate}
\qed
\end{proof}

We are now ready to bound the difference between the actual revenue and the smoothed revenue for any seller $i$ and price $p_i$.

\begin{lemma}
For any seller $i$ and fixed $\vec{p}_{-i}$ and for any fixed price $p$ chosen by seller $i$:
\[
	0 ~\leq~ \tilde{r}_i(p, \vec{p}_{-i}) ~-~ \tilde{r}^{sm}_i(p, \vec{p}_{-i}) ~\leq~   \epsilon r
\]
\end{lemma}

\begin{proof}
The left hand-side of the inequality follows directly from our construction of smoothed gradient. For the right-hand side we observe that the difference between the revenue values of the two curves is maximum at $p_i^*$. Hence, in the following, we shall focus on bounding $ \tilde{r}_i(p_{i}^*) ~-~ \tilde{r}^{sm}_i(p_{i}^*)$. Note that the gradient of the smoothed revenue function changes gradually from $-(E - 1)$ to $1$  in the price range $p_{i, \mathsf{X}_i}$ to $p_{i}^*$ and in the worse case, might change abruptly, i.e.
\begin{equation*}
\begin{aligned}
	\tilde{r}^{sm}_i( p^*) ~\geq~ & \enspace \tilde{r}^{sm}_i( p_{i, \mathsf{X}_i}) ~-~ (\tilde{p}_{i, \mathsf{X}_i} - \tilde{p}_{i}^*) \\[5pt]
							~\geq~ & \enspace \tilde{r}^{sm}_i( p_{i, \mathsf{X}_i}) ~-~ \frac{\epsilon r}{E}
\end{aligned}
\end{equation*}
\noindent
Using Lemma~\ref{lem:rev-bound} and using the fact that $\tilde{r}^{sm}_i( p_{i, \mathsf{X}_i}) = \tilde{r}_i( p_{i, \mathsf{X}_i})$:
\begin{equation*}
\begin{aligned}
	\tilde{r}_i(p_{i}^*) ~-~ \tilde{r}_i(p_{i, \mathsf{X}_i})  ~=~ & \enspace \frac{E - 1}{E} \cdot \epsilon r \\[6pt]
	 \tilde{r}_i(p_{i}^*) ~-~\left( \tilde{r}^{sm}_i( p_{i}^*) +  \frac{\epsilon r}{E} \right)  ~\leq~ & \enspace \frac{E - 1}{E} \cdot \epsilon r \\[6pt]
	 \tilde{r}_i(p_{i}^*) ~-~ \tilde{r}^{sm}_i(p_{i}^*) ~\leq~ &  \epsilon r
\end{aligned}
\end{equation*}
\qed
\end{proof}

\section{OMD satisfies DRVU Property}

\textbf{Optimistic Mirror Descent (OMD):} Consider the following online convex optimization problem: Let $\mathcal{F}$ be the convex set of actions of the learner. In each round $t$, the learner chooses an action $\vt{x}{t}{}$ and observes a linear utility function $\vec{u}^t$\footnote{For simplicity of presentation, we assume the utility function is linear}. The goal of the agent is to maximize her utility, i.e. $\sum_t \dott{\vt{x}{t}{}}{\vec{u}^t}$. Let $\mathcal{R}$ be a $1$-strongly convex function with respect to some norm $\norm{\cdot}$ on $\mathcal{F}$. Suppose the agent has a prediction, $M_t$, about the forthcoming utility vector in round $t$. The OMD algorithm incorporates this information into the decision process by the following interleaved sequence:
\[
	\vt{x}{t}{} = \argmin\limits_{\vec{x} \in \mathcal{F}} \enspace \eta_t \dott{\vec{x}}{M_t} + D_{\mathcal{R}}(\vec{x}, \vec{y}_{t-1}) \hspace{1cm} \vt{y}{t}{} = \argmin\limits_{\vec{y} \in \mathcal{F}} \enspace \eta_t \dott{\vec{y}}{\vec{u}^t} + D_{\mathcal{R}}(\vec{y}, \vec{y}_{t-1})
\]
where $D_{\mathcal{R}}$ is the Bregman divergence with respect to $\mathcal{R}$ and $ \{\eta_t \}$ is the sequence of step-sizes that can be chosen adaptively.

\begin{theorem}[Rakhlin and Sridharan~\cite{Rakhlin1}]
\label{thm:rakhlin}
The loss incurred by a learning agent in round $t$ under Optimistic Mirror Descent by choosing action $\vt{x}{t}{} \in \mathcal{F}$ with respect to any feasible strategy $\vt{x}{}{*}$ is upper bounded by:
\begin{align*}
	\dott{\vt{x}{t}{} - \vt{x}{}{*,t}}{\vec{u}^t} ~\leq~ \norm{\vec{u}^t - M_t} &\norm{\vt{x}{t}{} - \vt{y}{t}{}}  ~+~ \frac{1}{\eta} \left[ D_{\mathcal{R}}(\vt{x}{}{*,t} , \vt{y}{t-1}{}) - D_{\mathcal{R}}(\vt{x}{}{*,t} , \vt{y}{t}{}) \right]\\[5pt]
	& ~-~ \frac{1}{2\eta} \left[ \norm{\vt{x}{t}{} - \vt{y}{t}{}}^2 ~+~ \norm{\vt{x}{t}{} - \vt{y}{t-1}{}}^2 \right]
\end{align*}
\end{theorem}

\begin{fact}
\label{fact:f1}
For any $\rho > 0 $ and any numbers $a$ and $b$: $a \cdot b \leq \frac{\rho}{2} a^2 ~+~ \frac{1}{2 \rho} b^2$. 
\end{fact}

\begin{fact}
\label{fact:f2}
For any points $\vt{x}{}{}, \vt{y}{}{}, \vt{z}{}{} \in \mathcal{F}$, $$ D_{\mathcal{R}}(\vt{x}{}{} , \vt{z}{}{}) - D_{\mathcal{R}}(\vt{y}{}{} , \vt{z}{}{}) ~\leq~ D_{\mathcal{F}} \norm{\vt{x}{}{} - \vt{y}{}{} } $$ where $D_{\mathcal{F}} = \max\limits_{a, b \in \mathcal{F}} \norm{\vt{a}{}{} - \vt{b}{}{}}$.
\end{fact}

\begin{fact}
\label{fact:f3}
For any points $\vt{x}{t}{}, \vt{x}{t-1}{}, \vt{y}{0}{} \in \mathcal{F}$, $$ \sum\limits_{t} \norm{\vt{x}{t}{} - \vt{x}{t-1}{}}^2 ~\leq~ 2 \left( \sum\limits_{t} \norm{\vt{x}{t}{} - \vt{y}{0}{}}^2 + \norm{\vt{x}{t-1}{} - \vt{y}{0}{}}^2  \right)$$
\end{fact}
\noindent

\begin{theorem}
The dynamic regret of an agent under Optimistic Mirror Descent with $M_t = \vt{u}{}{t-1}$ with respect to the benchmark sequence of strategies $ \{ \vt{x}{}{*,t} \}_{t}$ is upper bounded by:
\[
	R_T \leq \frac{R}{\eta}  + \frac{D_{\mathcal{F}}}{\eta} \sum\limits_{t} \norm{\vt{x}{}{*,t} - \vt{x}{}{*,t-1} } + \eta \sum\limits_{t} \norm{\vec{u}^t - \vec{u}^{t-1}}_{*}^2 - \frac{1}{8 \eta}  \sum\limits_{t}  \norm{\vt{x}{t}{} - \vt{x}{t-1}{}}^2
\]
where $ R = \sup\limits_{x} D_{\mathcal{R}}(\vt{x}{}{}, \vt{y}{0}{})$ and $D_{\mathcal{F}} = \max\limits_{a, b \in \mathcal{F}} \norm{\vt{a}{}{} - \vt{b}{}{}}$.
\end{theorem}
\begin{proof}

By Theorem~\ref{thm:rakhlin} instantiated for $M_t = \vec{u}^{t-1}$, we have:
\begin{align*}
	\dott{\vt{x}{t}{} - \vt{x}{}{*,t}}{\vec{u}^t} ~\leq~ \norm{\vec{u}^t - \vec{u}^{t-1}} &\norm{\vt{x}{t}{} - \vt{y}{t}{}}  ~+~ \frac{1}{\eta} \left[ D_{\mathcal{R}}(\vt{x}{}{*,t} , \vt{y}{t-1}{}) - D_{\mathcal{R}}(\vt{x}{}{*,t} , \vt{y}{t}{}) \right]\\[5pt]
	& ~-~ \frac{1}{2\eta} \left[ \norm{\vt{x}{t}{} - \vt{y}{t}{}}^2 ~+~ \norm{\vt{x}{t}{} - \vt{y}{t-1}{}}^2 \right]
\end{align*}
\noindent
Using Fact~\ref{fact:f1} and by choosing $\rho = 2\eta$, we can bound the first part of the expression as:
\begin{equation}
\label{eqn:part1}
 \norm{\vec{u}^t - \vec{u}^{t-1}}_{*}  \norm{\vt{x}{t}{} - \vt{y}{t}{}} ~\leq~ \eta \norm{\vec{u}^t - \vec{u}^{t-1}}_{*}^2 + \frac{1}{4\eta}\norm{\vt{x}{t}{} - \vt{y}{t}{}}^2
\end{equation}
\noindent
Next, we can sum and rearrange the Bregman divergence terms to get:
\[
\sum\limits_{t=1}^T D_{\mathcal{R}}(\vt{x}{}{*,t} , \vt{y}{t-1}{}) - D_R(\vt{x}{}{*,t} , \vt{y}{t}{}) ~\leq~ R + \left( \sum\limits_{t} D_{\mathcal{R}}(\vt{x}{}{*,t} , \vt{y}{t-1}{}) - D_R(\vt{x}{}{*,t-1} , \vt{y}{t-1}{}) \right)
\] where $ R = \sup\limits_{x} D_R(\vt{x}{}{}, \vt{y}{0}{})$. Using Fact~\ref{fact:f2} in above inequality we get:
\begin{equation}
\label{eqn:part2}
	\sum\limits_{t=1}^T D_{\mathcal{R}}(\vt{x}{}{*,t} , \vt{y}{t-1}{}) - D_{\mathcal{R}}(\vt{x}{}{*,t} , \vt{y}{t}{}) ~\leq~ R + \sum\limits_{t} D \norm{\vt{x}{}{*,t} - \vt{x}{}{*,t-1} }
\end{equation}
\noindent
Finally, we bound the last part of the expression using Fact~\ref{fact:f3} and observing that in OMD algorithm we choose $\vt{x}{0}{} = \vt{y}{0}{} = \argmin\limits_{x} R(x)$. $$ \frac{1}{4\eta}  \left[ \norm{\vt{x}{t}{} - \vt{y}{t}{}}^2 + \norm{\vt{x}{t}{} - \vt{y}{t-1}{}}^2 \right] ~\geq~ \frac{1}{8\eta }  \sum\limits_{t} \norm{\vt{x}{t}{} - \vt{x}{t-1}{}}^2. $$

\end{proof}

\section{Revenue Optimization in Dynamic Markets}

We derive here a sequence of lemmas required to prove the final theorem.

\begin{lemma}
\label{lem:income-elasticity}
Let $\vec{p}^{old}$, $\vec{x}^{old}$ and $\vec{p}^{new}$, $\vec{x}^{new}$ denote the price and the resultant demand vectors respectively for some gross substitutes market.
\begin{enumerate}[(a)]
\item If $\vec{p}^{new} = \vec{p}^{old} (1+\epsilon)$ then, $ \vec{x}^{new} = \frac{\vec{x}^{old}}{1 + \epsilon} $.
\item If $\vec{p}^{new} = \frac{\vec{p}^{old}}{(1+\epsilon)}$ then, $ \vec{x}^{new} = \vec{x}^{old} ( 1 + \epsilon ) $.
\end{enumerate}

\end{lemma}

\begin{proof}
We only prove Part (a) here. Part (b) follows from identical steps. Note that increasing the prices of all items by a factor of $ (1+\epsilon)$ is equivalent to decreasing the income of all buyers by the same factor. Let the income of player $i$ be denoted by $I_i$. Then for any buyer $i$, $I^{new}_{i} = \frac{I^{old}_i}{(1 + \epsilon)}$. Further, for gross substitutes markets with CES utilities, it is known that the income elasticity parameter, $\epsilon_{I}$, for any player is exactly equal to 1. By definition of income elasticity:
\[
	\epsilon_{I} = \dfrac{\frac{\vec{x}^{new} - \vec{x}^{old}}{\vec{x}^{old}}}{\frac{I^{new} - I^{old}}{I^{old}}} = \frac{\frac{\vec{x}^{new} - \vec{x}^{old}}{\vec{x}^{old}}}{\frac{- \epsilon}{1 + \epsilon}} = 1
\]
Rearranging, $ \vec{x}^{new} = \frac{\vec{x}^{old}}{1 + \epsilon}$.
\qed
\end{proof}

\begin{lemma}
\label{lem:change_in_equi}
Suppose the supply vector changes from $\vec{w}_{old} = (w_i)_i$ to $\vec{w}_{new} = (w'_i, w_{-i})$. Let $\vec{p}^{eq,old}$ and $\vec{p}^{eq, new}$ be the equilibrium price vectors corresponding to the old and new supply vectors respectively.
\begin{enumerate}[(a)]
\item If $ w'_i = w_i \cdot \left( 1 + \epsilon \right) $ then,
$$ 1 ~\leq~ \max\limits_j \frac{p_j^{eq, old}}{p^{eq, new}_{j}} ~\leq~ (1 + \epsilon)$$
\item If $ w'_i = \frac{w_i}{1 + \epsilon} $ then,
$$ 1 ~\geq~ \min\limits_j \frac{p_j^{eq, old}}{p^{eq, new}_{j}} ~\geq~ \frac{1}{1 + \epsilon}  $$
\end{enumerate}

\end{lemma}
\begin{proof}
Consider the case where $w'_i = w_i \cdot \left( 1 + \epsilon \right)$. To prove contradiction, assume that for some player $j$, $\frac{p_j^{eq, old}}{p^{eq, new}_{j}}  = z ~>~ (1 + \epsilon)$. By equilibrium condition, 
\begin{equation}
\label{eqn:upper_bound}
\begin{aligned}
	w_j^{new} = & \enspace x_j \left( p_j^{eq, new}, p_{-j}^{eq, new} \right) \\[5pt]
	= & \enspace x_j \left( \frac{p_j^{eq, old}}{z},  p_{-j}^{eq, new} \right) \\[5pt]
	\stackrel{(a)}{\geq} & \enspace x_j \left( \frac{p_j^{eq, old}}{z},  \frac{p_{-j}^{eq, old}}{z} \right) \\[5pt]
	\stackrel{(b)}{=} & \enspace z \cdot 	x_j\left( \vec{p}^{eq, old} \right) = z \cdot  w_j^{old}.
\end{aligned}
\end{equation} The inequality $(a)$ follows from the definition of gross substitutes markets. Equality $(b)$ is the direct application of Lemma~\ref{lem:income-elasticity}. Since this is a contradiction, we conclude that $\max\limits_j \frac{p_j^{eq, old}}{p^{eq, new}_{j}} ~\leq~ (1 + \epsilon)$.
\vspace{5pt}

For the lower bound suppose that for some item $j$,  $\frac{p_j^{eq, old}}{p^{eq, new}_{j}}  = z_2 ~<~ 1 $. Then,
\begin{equation*}
\begin{aligned}
w_j^{new} =~ x_j(\vec{p}^{eq, new}) = & \enspace x_j\left( \frac{p^{eq, old}_{j}}{z_2}, p^{eq, new}_{-j} \right) \\[5pt]
\stackrel{(a)}{\leq} &  \enspace x_j\left( \frac{p^{eq, old}_{j}}{z_2},  \frac{p^{eq, old}_{-j}}{z_2} \right) \\[5pt]
\stackrel{(b)}{=}  & \enspace x_j(\vec{p}^{eq, old}) \cdot z_2  =  \enspace w_j^{old} \cdot z_2
\end{aligned}
\end{equation*}
which is a contradiction for any $z_2 < 1$. The inequalities $(a)$ and $(b)$ follow the same reasoning as in inequality~\ref{eqn:upper_bound}. This implies that for an increase in supply of item $i$, the price of no item $j$ increases and the maximum decrease in the price of any item $j$ is at most a factor of $(1 + \epsilon)$. By analogous arguments, we can prove the result for the case when the supply decreases.
\qed
\end{proof}

\begin{cor}
\label{corr: equi-change}
Let $\norm{\cdot}_{1}$ denote $1$-norm. If the supply vector changes from $\vec{w}^{t}$ to $\vec{w}^{t+1}$, where the supply of each item may change independently then:
\[
	\max\limits_{j} | \tilde{p}_j^{eq, t+1} - \tilde{p}_j^{eq, t} | ~\leq~ \norm{\tilde{\vec{w}}^{t+1} - \tilde{\vec{w}}^{t}}_{1}
\]
\end{cor}
\begin{proof}
First note that we can re-write the result of Lemma~\ref{lem:change_in_equi} in log scale as: 
$$\max\limits_{j} | \tilde{p}_j^{eq, t+1} - \tilde{p}_j^{eq, t} |  ~\leq~   | \tilde{w}^{t+1}_i - \tilde{w}^{t}_i |, $$
where we assumed that the supply of only item $i$ changed. Now, for any two supply vectors $\vec{w}^{t+1}$ and $\vec{w}^{t}$, consider the switch from $\vec{w}^{t}$ to $\vec{w}^{t+1}$ sequentially in a pre-defined order while keeping the supplies of remaining sellers fixed during this switch. From Lemma~\ref{lem:change_in_equi}, we know that for each such intermediate step, where the supply of only item $j$ changes, the maximum change in equilibrium is at most $|\tilde{w}^{t+1}_j - \tilde{w}^{t}_j|$. The cumulative change in equilibrium can then simply be upper bounded by the sum of these individual changes.
\qed
\end{proof}

\begin{theorem}
Let $W_T = \sum_{t} \norm{\tilde{\vec{w}}^t - \tilde{\vec{w}}^{t-1}}_{1}$ denote the cumulative change in the market in terms of changes in supplies. If each seller $i$ uses OMD algorithm on the log-revenue function with smoothed gradient feedback and threshold demand $\mathsf{X}_i^t ~=~ \frac{w_i^t}{\exp\left( \epsilon r\right)}$.  Let $\{ p_i^{eq,t} \}_{t}$ denote the sequence of equilibrium prices for seller $i$ then:
\[
 \sum\limits_{t=1}^T  \left( 1 - \epsilon R \right) r_i^t(p_i^{eq,t}) - r_i^t(p^t) ~\leq~ O\left( \left(  \frac{R^2 E^2}{\epsilon r} \right)^{1/2} \cdot (1 + W_T) T^{1/4} \right)
\]
\end{theorem}
\begin{proof}
We can obtain this bound using almost the same steps as in Theorem \ref{thm:main1} and using Corollary~\ref{corr: equi-change} to account for the cumulative change in benchmark prices.
\qed
\end{proof}

\end{document}